\def\appendixurl{https://github.com/yetigurbuz/xml-dml/blob/main/appendix.pdf}

\documentclass{article}

\usepackage{spconf}

\usepackage{import}
\usepackage[utf8]{inputenc} 
\usepackage[T1]{fontenc}    
\usepackage{lmodern}

\usepackage[pdftex]{graphicx}
\graphicspath{{./figures/}} 
\DeclareGraphicsExtensions{.pdf,.jpeg,.png} 
\usepackage{wrapfig}
\usepackage{epsfig}
\usepackage{tabularx}
\usepackage{booktabs, multicol, multirow}	
\usepackage[table, xcdraw]{xcolor}
\usepackage{color}
\usepackage{array}
\usepackage{caption}
\usepackage{algorithmicx}
\usepackage[noend]{algpseudocode}
\usepackage{eqparbox}
\usepackage{algorithm}

\usepackage{amsmath}
\usepackage{mathtools}
\usepackage{commath}
\usepackage{mathrsfs}
\usepackage{amssymb}
\usepackage{amsfonts}       
\usepackage{dsfont}
\usepackage{bbm}
\usepackage{bm}
\usepackage{fancyhdr}
\usepackage{cancel}
\usepackage{nicefrac}       
\usepackage{microtype}      
\usepackage{mdwmath}
\usepackage{mdwtab}
\usepackage{wasysym} 

\usepackage{hyperref} 

\usepackage[capitalize]{cleveref}
\crefname{section}{\S}{\S\S}
\Crefname{section}{Section}{Sections}
\Crefname{table}{Table}{Tables}
\crefname{table}{Tab.}{Tabs.}

  \usepackage{cite}

\usepackage{url}            

\usepackage{comment}

\DeclareMathOperator*{\argmin}{arg\,min}
\DeclareMathOperator*{\argmax}{arg\,max}

\newcommand{\R}{\ensuremath{\mathbb{R}}}

\DeclareMathSymbol{\shortminus}{\mathbin}{AMSa}{"39}

\newcommand{\sxtimes}{\mathsf{x}}
\newcommand{\T}{\ensuremath{^\intercal}}
\newcommand{\ie}{\textit{i}.\textit{e}., }
\newcommand{\eg}{\textit{e}.\textit{g}., }

\newenvironment{proof}[1]{\par\noindent\underline{Proof:}\space#1}{\hfill $\blacksquare$}

\newtheorem{claim}{Claim}[section]

\newtheorem{definition}{Definition}[section]
\newtheorem{lemma}{Lemma}[section]

\newtheorem{apclaim}{Claim}

\newtheorem{apdefinition}{Definition}
\newtheorem{aplemma}{Lemma}

\numberwithin{equation}{section}

\makeatletter
\def\BState{\State\hskip-\ALG@thistlm}
\makeatother

\usepackage{xspace}
\makeatletter
\DeclareRobustCommand\onedot{\futurelet\@let@token\@onedot}
\def\@onedot{\ifx\@let@token.\else.\null\fi\xspace}

\def\eg{\emph{e.g}\onedot} 
\def\ie{\emph{i.e}\onedot}

\makeatother

\def\yeti{Yeti~Z.~G\"{u}rb\"{u}z}

\def\aydin{A.~Ayd{\i}n~Alatan}
\def\inst1{Dept. of Elect. and Elec. Eng., Middle East Technical University, Ankara, Turkey}

\def\email1{yeti@metu.edu.tr}
\def\email2{ogul.can@metu.edu.tr}
\def\email3{alatan@metu.edu.tr}
\def\email4{ada.gorgun@metu.edu.tr}

\name{\yeti\textsuperscript{$\dagger$}\thanks{\textsuperscript{$\dagger$}Affiliated with OGAM-METU during the research.} \qquad \aydin\textsuperscript{$\ddagger$}}
\address{\textsuperscript{$\dagger$}RSiM, Technische Universit\"at Berlin, DE\quad\textsuperscript{$\ddagger$}Center for Image Analysis (OGAM), METU, TR}

\begin{document}

\title{Generalizable Embeddings with Cross-batch Metric Learning}
\maketitle

\begin{abstract}
Global average pooling (GAP) is a popular component in deep metric learning (DML) for aggregating features. Its effectiveness is often attributed to treating each feature vector as a distinct semantic entity and GAP as a combination of them. Albeit substantiated, such an explanation's algorithmic implications to learn generalizable entities to represent unseen classes, a crucial DML goal, remain unclear. To address this, we formulate GAP as a convex combination of learnable \textit{prototypes}. We then show that the prototype learning can be expressed as a recursive process fitting a \textit{linear predictor} to a batch of samples. Building on that perspective, we consider two batches of \textit{disjoint} classes at each iteration and regularize the learning by expressing the samples of a batch with the prototypes that are fitted to the \textit{other batch}. We validate our approach on 4 popular DML benchmarks.
\end{abstract}


\fancypagestyle{firststyle}
{
   \fancyhead{}
   \lhead{Accepted as a conference paper at ICIP 2023}
   \renewcommand{\headrulewidth}{0pt} 
}
\thispagestyle{firststyle}

\section{Introduction}
\label{sec:introduction}
Deep metric learning (DML) considers image-label pairs $(I,L)$ and aims to learn an embedding function $I\rightarrow y$ that maps images $I$ to vectors $y$ such that the Euclidean distance in the space of embeddings is consistent with the label information. More specifically, $\Vert y_i \shortminus y_j\Vert_2$ is small whenever $L_i=L_j$, and large whenever $L_i \neq L_j$. To enable learning, this requirement is represented via loss function $\ell((y_i, L_i), (y_j,L_j))$ (\eg, \textit{contrastive} \cite{wu2017sampling}, \textit{triplet} \cite{schroff2015facenet}, \textit{multi-similarity} \cite{Wang_2019_CVPR_MS}) and the typical learning mechanism is gradient descent of an empirical risk function defined over a batch of data points: $\mathcal{L}_{\text{DML}} \coloneqq \Sigma_{ij}\ell((y_i, L_i), (y_j,L_j))$.

 Primary thrusts in DML include tailoring pairwise loss terms \cite{musgrave2020metric}, pair mining \cite{roth2020revisiting} and data augmentation with either synthesizing informative samples \cite{venkataramanan2022it} or with virtual embeddings called \textit{proxies} \cite{kim2020proxy, teh2020proxynca++}. To improve generalization; training strategies upon characterization of the generalization bounds \cite{gurbuz2021asap}, separating unique and shared characteristics among classes,
 \cite{Roth_2019_ICCV}, intra-batch feature aggregation \cite{intrabatch}, ranking surrogates \cite{profs,patel2022recall}, further regularization terms \cite{roth2022non}, and various architectural designs such as ensemble \cite{zheng2021deep} and multi-task \cite{milbich2020diva,roth2021s2sd} models are utilized in the prolific DML literature. A shared component of these diverse methods is the embedding function which is a convolutional neural network (CNN) followed by global average pooling (GAP) \cite{musgrave2020metric}.

\begin{figure}[t]
  \centering
  \centerline{\includegraphics[width=0.782\linewidth]{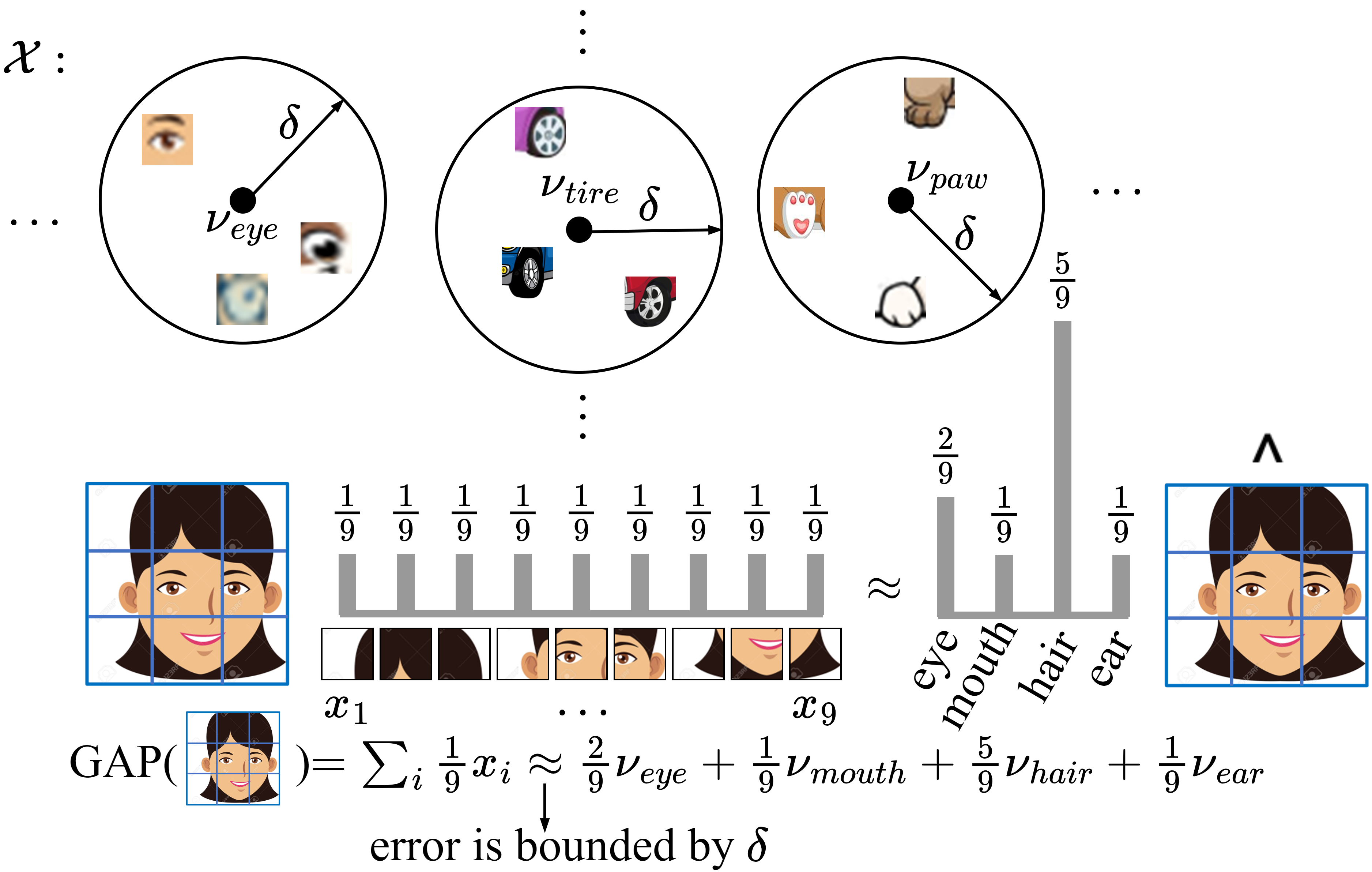}}
  \caption{Visualization of \cref{thm:gap_proto} on which our method is built. } 
	\label{fig:motive}
  \end{figure}

%

Though simple, GAP is a highly effective way to aggregate information. Empirically validated folklore \cite[and references therein]{zhou2018interpreting} to explain the effectiveness of GAP is considering each pixel of the CNN feature map as corresponding to a separate semantic entity and GAP as the combination of them \cite{gurbuz2019novel}. A critical desiderata of DML is generalizing the learned embedding function to unseen classes. Thus, the learned semantic entities should be able to express novel classes, \eg, learning \emph{"tire"} and \emph{"window"} to represent \emph{"car"} instead of learning \emph{"car"}. However, no explicit mechanism exists in the current DML approaches to enforce this behavior. Moreover, supervised DML losses provide guidance for seen classes, that yields entities fitted to classes and possibly hinders generalization capability. In this paper, we address explicitly learning generalizable semantic entities in the context of GAP.  

Briefly, our contributions include the following; $i)$ we formulate GAP as a convex combination of learnable prototypes (\cref{fig:motive}) to enable explicit learning of the semantic entities, $ii)$ we show that the prototype learning can be expressed as a recursive process fitting a linear predictor to the batch of samples, and $iii)$ we tailor a regularization loss (\cref{fig:method}) built on expressing the set of classes with the prototypes fitted to another set of classes. Through rigorous experimentation, we validate our theoretical claims and demonstrate the effectiveness of our approach.

\begin{figure*}[!ht]
  \centering
  \centerline{\includegraphics[width=1.\linewidth]{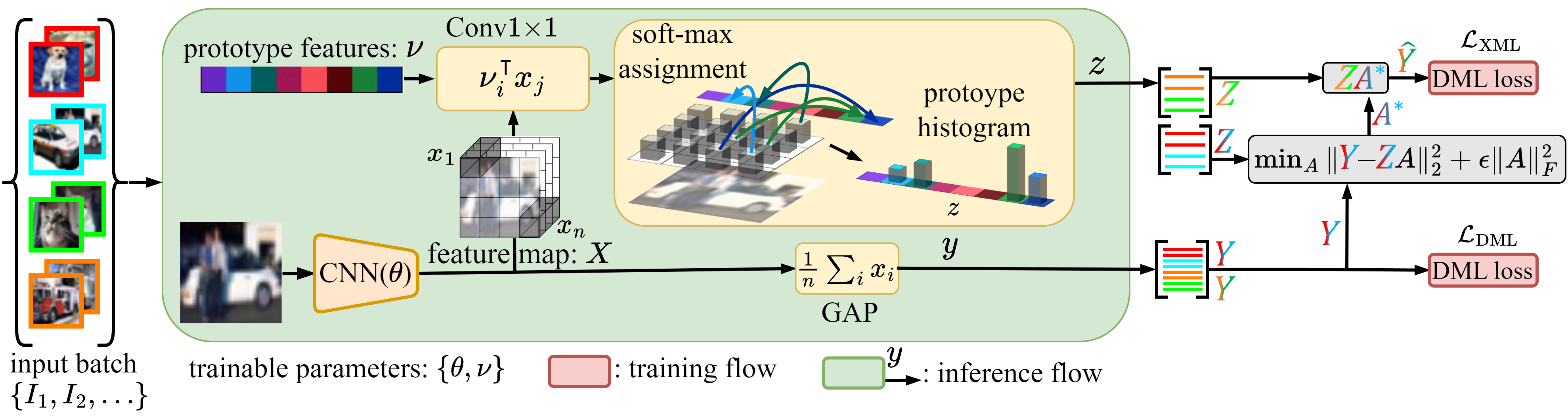}}
  \caption{Embodiment of the method, where the histograms $Z{=}[z_i]_i$  \eqref{eq:trainable_hist_operator} and GAP vectors $Y{=}[y_i]_i$ are coloured w.r.t. their class label.}
	\label{fig:method}
  \end{figure*}

%

\section{Method}
We propose a regularization loss to learn transferable features. Our loss is built on solving a metric learning problem on a batch and then evaluate the learned metric on another batch of unseen classes. We first express GAP as a convex combination of learnable prototypes in \cref{sec:gap_pcc}. We then associate prototype learning with a recursive process fitting a linear predictor to a batch of samples in \cref{sec:rls}. Building on that, we formulate our loss in \cref{sec:xml}. We defer all the upcoming proofs to \hyperref[sec:appendix]{appendix}.


\subsection{GAP as Convex Combination of Prototypes}
\label{sec:gap_pcc}
We consider embedding functions that are implemented as CNN followed by GAP, \ie, $I\overset{\text{CNN}}{\longrightarrow}X\overset{\nicefrac{1}{n}\Sigma_i x_i}{\longrightarrow}y$ where $X$ is $w{\sxtimes}h$ feature map and $n=w{\sxtimes}h$. We introduce the following operator to compose a histogram representation from the collection of features.

\begin{definition}[\textbf{Histogram Operator}] \label{def:hist_operator} For $n$-many $d$-dimensional features $X{=}[ x_i{\in}\R^d ]_{i=1}^n$ and $m$-many prototype features $\mathcal{V}=[ \nu_i{\in}\R^d ]_{i=1}^m$ of the same dimension, the histogram of $X$ on $\mathcal{V}$ is denoted as $z^\ast$ which is computed as the minimizer of the following problem:
\begin{equation} \label{eq:hist_operator}
(z^\ast,\pi^\ast) =   \!\argmax_{z\in\mathcal{S}^m, \pi\geqslant0}
 \!\textstyle\sum_{ij}\nu_i\T x_j\pi_{ij} \text{ s.to }
{\small\begin{array}[t]{l} \Sigma_i\pi_{ij}{=}\nicefrac{1}{n}\\[-1ex]\Sigma_{j}\pi_{ij}{=}z_i \end{array}}
\end{equation}
where $\mathcal{S}^m \coloneqq \lbrace p \in \R^{m}_{\geqslant 0}\mid \Sigma_i p_i = 1  \rbrace$.
\end{definition}

\begin{claim}
The solution of the problem in \eqref{eq:hist_operator} reads:
\begin{equation}
    \pi^\ast_{ij} = \nicefrac{1}{n}\mathds{1}({i = \mathrm{argmax}_k\lbrace \nu_k\T x_j \rbrace})
\end{equation}
where $\mathds{1}(c)$ is 1 whenever $c$ is true and 0 otherwise.
\end{claim}
In words, \emph{histogram operator} basically assigns each feature to their \textit{nearest} prototype and accumulates $\nicefrac{1}{n}$ mass for each assigned feature. 
We now consider a set of prototypes in the feature space $\mathcal{X}$ where the convolutional features $x_i$ lie. We consider $m$-many prototype features $\mathcal{V} = \lbrace \nu_i\rbrace_{i=1}^m$ so that the set $\mathcal{V}$ is $\delta$-cover of the feature space, $\mathcal{X}$. Namely, for any $x\in\mathcal{X}$, we have a prototype $\nu_x$ such that $\Vert x  \shortminus \nu_x\Vert_2 \leqslant \delta$.  

Given $n$-many convolutional features $X=[ x_i]_{i=1}^n$ we compute the histogram of $X$ on $\mathcal{V}$ (\ie, $z^\ast$) using \eqref{eq:hist_operator} and obtain global representation $\hat{y}$ as:
\begin{equation}\label{eq:pcc}
    \hat{y} = \textstyle\sum_{k=1}^m z^\ast_k \nu_k\quad.
\end{equation}
Note that GAP representation is $y=\nicefrac{1}{n}\Sigma_{i=1}^n x_i$. By the following lemma, we show that GAP  is \textit{approximately} equivalent to $\hat{y}$, \ie, \textit{convex combination of prototypes}. 

\begin{lemma}\label{thm:gap_proto}
Given $n$-many convolutional features $X=[ x_i{\in}\mathcal{X}]_{i=1}^n$ and $m$-many prototype features $\mathcal{V}=[ \nu_i]_{i=1}^m$ with $\lbrace \nu_i\rbrace_{i=1}^m$ being $\delta$-cover of $\mathcal{X}$. If $z^\ast$ is the histogram  of $X$ on $V$, defined in \eqref{eq:hist_operator}, then we have:
$$ \Vert \textstyle\sum_{i=1}^m z^\ast_i\nu_i \shortminus \textstyle\sum_{j=1}^n \tfrac{1}{n} x_j \Vert_2 \leqslant \delta $$
\end{lemma}
We visualize the result of \Cref{thm:gap_proto} in \cref{fig:motive}, which implies that with GAP each image is represented as the convex combination of the prototype vectors. To generalize DML to unseen classes, we want the prototypes to represent transferable entities such as \emph{"tire"} and \emph{"window"} rather than classes themselves (\eg, \emph{"car"}). To enforce that, we first formulate \emph{histogram operator} as a trainable layer by smoothing the objective of \eqref{eq:hist_operator} with entropy:
\begin{equation} \label{eq:trainable_hist_operator}
(z^\prime,\pi^\prime) =   \!\argmax_{\substack{\Sigma_i\pi_{ij}{=}\nicefrac{1}{n}\\\Sigma_{j}\pi_{ij}{=}z_i\\z\in\mathcal{S}^m, \pi>0}}
 \!\textstyle\sum_{ij}\nu_i\T x_j\pi_{ij} - \tfrac{1}{\varepsilon}\textstyle\sum_{ij}\pi_{ij}\log \pi_{ij}
\end{equation}
which admits \emph{soft-max} solution as: $z^\prime_i=\tfrac{1}{n}\Sigma_j\tfrac{\mathrm{exp}(\varepsilon\nu_i\T x_j)}{\Sigma_k\mathrm{exp(\varepsilon\nu_k\T x_j)}}$. Thus, it can be implemented with $1{\sxtimes}1$ convolution and \emph{soft-max} layers (\cref{fig:method}). In the following sections, we derive a loss to regularize the learning of the prototypes.

\subsection{Learning the Prototypes}
\label{sec:rls}
Given $Z{=}[z_i]_i$ and $Y{=}[y_i]_i$ denoting the histograms obtained by \eqref{eq:trainable_hist_operator} and GAP representations of a batch, respectively, we can learn the prototypes jointly with the embedding function by augmenting $\Vert \mathcal{V}Z \shortminus Y \Vert_F^2$ to the DML loss. However, that does not guarantee transferable representations. We now alternatively express the learning mechanism of the prototypes as a recursive process and derive a loss to regularize the learning.

Let $(Z_1,Y_1), (Z_2,Y_2),\ldots, (Z_K,Y_K)$ be the representations we obtain during the course of $K$-step training. We can obtain $\mathcal{V}^{(K)}$, \ie, the prototypes at $K$, as the solution of the following problem:
\begin{equation}
    \mathcal{V}^{(K)}=\argmin_A\textstyle\sum_{i=1}^K \alpha^{K-i}\Vert A\,Z_i - Y_i\Vert_F^2 + \beta\Vert A\Vert_F^2
\end{equation}
where $0<\alpha\leqslant 1$ is the forgetting factor to put more emphasis on the recent representations, and $\beta\Vert A\Vert_F^2$ is to improve robustness. We can obtain the solution as \cite{hayes1996statistical}:
\begin{equation}
    \mathcal{V}^{(K)} = R_K^{\shortminus 1} Q_K
\end{equation}
where $R_K =\Sigma_i\alpha^{K-i}Z_iZ_i\T+\beta I$ and $Q_K = \Sigma_i\alpha^{K-i}Z_iY_i\T$. For a new batch $(Z,Y)$ at step $K+1$, we can update the solution as:
\begin{equation}\label{eq:rls_proto}
    \mathcal{V}^{(K+1)} = W_K\mathcal{V}^{(K)} + (I - W_K)\mathcal{V}
\end{equation}
where $\mathcal{V} = \argmin_A \Vert A\,Z-Y\Vert^2_F+(1\shortminus \alpha)\beta\Vert A\Vert_F^2$ is the prototypes fitted to the current batch as $\mathcal{V}=R^{\shortminus 1}Z\,Y\T$ with $R^{\shortminus 1}=Z\,Z\T+(1\shortminus\alpha)\beta I$, and $W_K=R^{\shortminus 1}(R_K^{\shortminus 1}+\alpha R^{\shortminus 1})^{\shortminus 1}$. The results mainly come from \emph{Woodbury identity} similar to derivation of \emph{RLS} filter \cite{hayes1996statistical}. 

Practically, learning prototypes with gradient descent of $\Vert \mathcal{V}Z \shortminus Y \Vert_F^2$ is more appealing. That said, the form of the recursive update in \eqref{eq:rls_proto} reveals that the learned prototypes are the \emph{weighted} combinations of the prototypes fitted to the batch of samples. Thus, imposing constraints on per-batch-fitted prototypes can be a decisive step to obtain a batch-based regularization loss. In the following section, we build on that perspective to formulate our loss to regularize prototype learning.

\subsection{Cross-batch Metric Learning}
\label{sec:xml}
The formulation in \eqref{eq:rls_proto} reinterprets the learning mechanism of prototypes, that is based on iteratively fitting prototypes to batch of samples $(Z, Y)$ as:
\begin{equation}\label{eq:batch_ml}
    \mathcal{V} = \argmin_A \Vert A\,Z - Y \Vert_F^2 + \epsilon \Vert A \Vert_F^2
\end{equation}
Assuming that representations in $Y$ are consistent with the label information, expression in \eqref{eq:batch_ml} is equivalent to solving a metric learning problem for $(Z,Y)$ tuples \cite{perrot2015regressive}. We now exploit this observation to derive our loss.

We first split the batch $(Z, Y)$ into two as $(Z_1, Y_1)$ and $(Z_2, Y_2)$ such that class sets of the two batches are disjoint. Similar to \eqref{eq:rls_proto}, we express \eqref{eq:batch_ml} as:
\begin{equation}
    \mathcal{V} = W\mathcal{V}_1 + (I-W)\mathcal{V}_2
\end{equation}
where $\mathcal{V}_k = \argmin_A \Vert A\,Z_k - Y_k \Vert_F^2 + \nicefrac{\epsilon}{2} \Vert A \Vert_F^2$, and $W=R_2^{\shortminus 1}(R_1^{\shortminus 1}+ R_1^{\shortminus 2})^{\shortminus 1}$ with $R_k = Z_k\,Z_k^T + \nicefrac{\epsilon}{2}I$. Hence, we express the learning mechanism at each batch as the \emph{weighted} combination of the two metrics fitted to the different sets of classes, that sets the stage for the rest of the formulation. 

Consider the prototypes $\mathcal{V}_1$ fitted to $(Z_1,Y_1)$. If those prototypes, $\mathcal{V}_1$, are corresponding to transferable entities, then their combination with the weights in $Z_2$ should yield embeddings that are consistent with the label information. Specifically, $\hat{Y}_2 = \mathcal{V}_1\,Z_2$ should also minimize DML loss.

Formally, given a batch $(Z{=}[Z_1\,Z_2], Y{=}[Y_1\,Y_2])$, we first obtain the prototypes as $\mathcal{V}_k=(Z_k\,Z_k\T+\epsilon I)^{\shortminus 1}Z_k\,Y_k\T$ for $k{\in}\lbrace1,2\rbrace$ or equivalently $\mathcal{V}_k=Y_k(Z_k\T Z_k+\epsilon I)^{\shortminus 1}Z_k\T$, if the batch size is less than the number of prototypes for computational efficiency. Given a DML loss function $\ell((y_i,L_i), (y_j,L_j))$, \eg, \emph{contrastive} \cite{wu2017sampling}, we formulate our \emph{cross-batch metric learning} (XML) loss as:
\begin{equation}\label{eq:xml_loss}
\mathcal{L}_{\text{XML}} = \textstyle\sum\limits_{k}\textstyle\sum\limits_{\hat{y}_i,\hat{y}_j\in \hat{Y}_k}\ell((\hat{y}_i,L_i), (\hat{y}_j,L_j))
\end{equation}
 for $k{=}1,2$ where $\hat{Y}_1 = \mathcal{V}_2\,Z_1$ and $\hat{Y}_2 = \mathcal{V}_1\,Z_2$. In words, we solve a metric learning problem for a set of classes and then compute its performance on another set of unseen classes. Having closed form solution for $\mathcal{V}_k$ in terms of $(Z_k, Y_k)$ enables us to express the metric learning problem as a differentiable operation. Hence, unseen class performance can be explicitly enforced through a batch-based loss term (\ie,  $\mathcal{L}_{\text{XML}}$) that can be jointly optimized with gradient descent of any DML loss. In particular, we combine this loss with the metric learning loss as:
 \begin{equation}
     \mathcal{L} = (1{\shortminus}\lambda)\mathcal{L}_{\text{DML}} + \lambda\mathcal{L}_{\text{XML}}
 \end{equation}
The proposed loss assesses the unseen class generalization performance of locally fitted prototypes. Intuitively, such a regularization in learning should be useful in better generalization of the CNN features as well as GAP embeddings since prototypes are connected to CNN features and GAP embeddings through analytical operations.

\section{Experimental Work}
We start our empirical study with evaluations on DML benchmarks to show the effectiveness of XML. We extend our study further to validate the role of XML in learning. 

\subsection{Deep Metric Learning Experiments}
\begin{table}[t]
\centering
\caption{Evaluation on 4 DML benchmarks with 2 widely-acknowledged settings. Red: the best. Blue: the second best. Bold: improvement over baseline.}
\label{tab:dml_eval}
\resizebox{\columnwidth}{!}{%
\begin{tabular}{@{}lcccccccc@{}}
\toprule
Dataset$\rightarrow$        & \multicolumn{2}{c}{\textbf{SOP}}                                              & \multicolumn{2}{c}{\textbf{InShop}}                                           & \multicolumn{2}{c}{\textbf{CUB}}                                              & \multicolumn{2}{c}{\textbf{Cars}}                                             \\ \cmidrule(lr){2-3} \cmidrule(lr){4-5} \cmidrule(lr){6-7} \cmidrule(lr){8-9} 
Setting$\rightarrow$        & \multicolumn{8}{c}{\textbf{BNInception - MLRC Evaluation (MAP@R)}}                                                                                                                                                                                                                                                            \\ \cmidrule(l){2-9} 
\textbf{Method}$\downarrow$ & 512D                                  & 128D                                  & 512D                                  & 128D                                  & 512D                                  & 128D                                  & 512D                                  & 128D                                  \\ \midrule
Triplet\cite{schroff2015facenet}                     & 45.88                                 & 40.01                                 & {\color[HTML]{3166FF} \textbf{59.67}} & 54.25                                 & 23.65                                 & 18.54                                 & 22.67                                 & 15.74                                 \\
MS\cite{Wang_2019_CVPR_MS}                          & 44.19                                 & 40.34                                 & 58.79                                 & 54.85                                 & 24.95                                 & 20.13                                 & {\color[HTML]{3166FF} \textbf{27.16}} & 18.73                                 \\
PNCA++\cite{teh2020proxynca++}                      & 47.11                                 & 43.57                                 & 57.58                                 & 54.41                                 & 25.27                                 & 20.63                                 & 26.21                                 & 18.61                                 \\ \midrule
\textbf{C}ontrastive\cite{wu2017sampling}                 & 45.85                                 & 41.79                                 & 59.07                                 & 55.38                                 & 25.95                                 & 20.58                                 & 24.38                                 & 17.02                                 \\
\textbf{C+XML}              & \textbf{46.84}                        & \textbf{42.73}                        & {\color[HTML]{FE0000} \textbf{59.75}} & {\color[HTML]{FE0000} \textbf{55.63}} & {\color[HTML]{FE0000} \textbf{27.58}} & {\color[HTML]{FE0000} \textbf{22.03}} & \textbf{26.33}                        & \textbf{18.31}                        \\ \midrule
\textbf{PA}nchor\cite{kim2020proxy}                     & {\color[HTML]{3166FF} \textbf{48.08}} & {\color[HTML]{3166FF} \textbf{44.33}} & 58.02                                 & 54.98                                 & 26.20                                 & 20.94                                 & 27.14                                 & {\color[HTML]{3166FF} \textbf{19.15}} \\
\textbf{PA+XML}             & {\color[HTML]{FE0000} \textbf{49.16}} & {\color[HTML]{FE0000} \textbf{45.15}} & \textbf{58.66}                        & {\color[HTML]{3166FF} \textbf{55.46}} & {\color[HTML]{3166FF} \textbf{26.51}} & {\color[HTML]{3166FF} \textbf{21.34}} & \textbf{27.56}                        & {\color[HTML]{FE0000} \textbf{19.28}} \\ \midrule
Setting$\rightarrow$        & \multicolumn{8}{c}{\textbf{ResNet50 - Conventional Evaluation (R@1)}}                                                                                                                                                                                                                                                         \\ \midrule
PAnchor\cite{kim2020proxy}                     & \multicolumn{2}{c}{80.00}                                                     & \multicolumn{2}{c}{92.10}                                                     & \multicolumn{2}{c}{69.70}                                                     & \multicolumn{2}{c}{87.70}                                                     \\
MS+Metrix\cite{venkataramanan2022it}                   & \multicolumn{2}{c}{81.00}                     & \multicolumn{2}{c}{92.20}                     & \multicolumn{2}{c}{{\color[HTML]{3166FF} \textbf{71.40}}}                                                     & \multicolumn{2}{c}{{\color[HTML]{FE0000} \textbf{89.60}}}                                                     \\ \midrule
LIBC\cite{intrabatch}                        & \multicolumn{2}{c}{{\color[HTML]{3166FF} \textbf{81.40}}}                                                     & \multicolumn{2}{c}{{\color[HTML]{3166FF} \textbf{92.80}}}                                                     & \multicolumn{2}{c}{70.30}                     & \multicolumn{2}{c}{88.10}                     \\
LIBC\textbf{+XML}           & \multicolumn{2}{c}{{\color[HTML]{FE0000} \textbf{81.48}}}                     & \multicolumn{2}{c}{{\color[HTML]{FE0000} \textbf{93.04}}}                     & \multicolumn{2}{c}{{\color[HTML]{FE0000} \textbf{70.49}}}                     & \multicolumn{2}{c}{{\color[HTML]{3166FF} \textbf{88.38}}}                     \\ \bottomrule
\end{tabular}%
}
\end{table}

\textbf{Setup.} We evaluate our method on CUB \cite{wah2011caltech}, Cars \cite{krause2014submodular}, InShop \cite{liu2016deepfashion}, and SOP \cite{oh2016deep}. Minimizing the confounding of factors other than our proposed method, we keep the comparisons as fair as possible following the MLRC  \cite{musgrave2020metric} procedures with BNInception embeddings \cite{normalization2015accelerating}. We additionally evaluate XML following the conventional settings \cite{intrabatch} with ResNet50 \cite{he2016identity} embeddings. For XML, $\varepsilon{=}10$, $\lambda{=}0.01$, $\epsilon{=}0.05$, and $m{=}64$ in CUB\&Cars, and $m{=}128$ in SOP\&InShop, based on our empirical analysis.  

\textbf{Results.} We apply XML with \emph{contrastive} \cite{wu2017sampling}  (C+ XML) and \emph{ProxyAnchor} \cite{kim2020proxy} (PA+XML) losses in MLRC setting, and with \emph{LIBC} \cite{intrabatch} in conventional setting. For MLRC, we report average (128D) and concatenated (512D) model MAP@R \cite{musgrave2020metric} performance, and R@1 for the conventional evaluation in \cref{tab:dml_eval} (higher the better). We observe consistent improvements upon direct application of DML losses in all datasets and boost state-of-the-art.

\subsection{Proof of the Concept}
\label{sec:ablation}
For the following, we perform DML trainings with XML on Cifar10 \cite{cifar} dataset using ResNet20 \cite{he2016identity} architecture.

\begin{figure}[!ht]
  \centering
  \centerline{\includegraphics[width=1.0\linewidth,keepaspectratio]{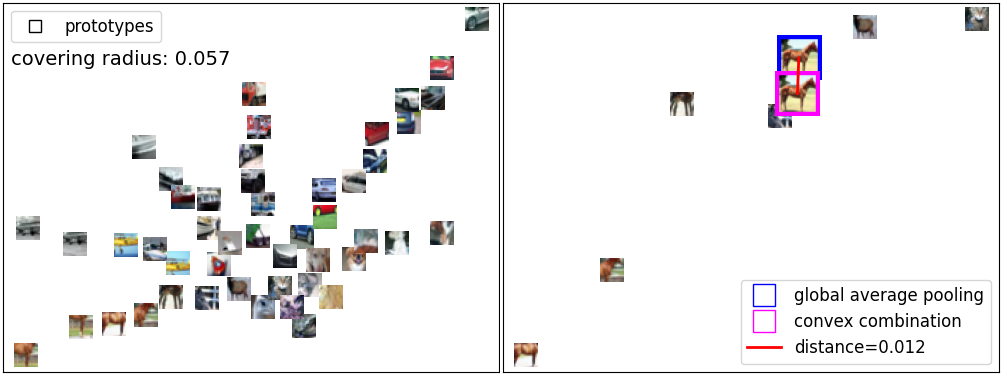}}
\caption{Prototypes with their covering radius (left), and  GAP and PCC embedding of a sample (right) with the assigned prototypes.}
\label{fig:prototypes}
\end{figure}

\textbf{GAP and prototypes.} To empirically verify  \cref{thm:gap_proto}, we use 2D feature embeddings for direct visualization. We sample 64 images from each class and obtain the local CNN features as well as the GAP features. We compute $48$-many prototypes among the local features using \textit{greedy k-center} \cite{gurbuz2021asap}. We plot the prototypes in \cref{fig:prototypes} where we see that prototypes correspond to generalizable semantic entities. We also provide the covering radius (\ie, $\delta$ in $\delta$-cover) of the prototype set and the discrepancy between GAP and prototype convex combination (PCC) embeddings, which is less than $\delta$ as \cref{thm:gap_proto} claims.

\textbf{Prototypes with XML}. We test the impact of XML on learned prototypes by performing DML on Cifar10 with 8 prototypes. We compare results with and without $\mathcal{L}_{\text{XML}}$ and visualize the prototype histograms for each class in \cref{fig:xml}. With $\mathcal{L}_{\text{XML}}$, we observe transferable representations and that the prototypes are fit to transferable entities while they are fit to classes without it. For instance, XML prototypes represent a \emph{"car"} in terms of parts and use some of them in the representation of \emph{"cat"} as well. We quantitatively evaluate this behavior by randomly splitting the classes in half and using cross-batch metric learning in \cref{sec:xml}. Our evaluation shows that the features and prototypes with XML have superior unseen class generalization ($MAP_x$) while the seen class performances ($MAP_c$) are similar. We repeated the experiment 1000 times to ensure validity.

\begin{figure}[!t]
  \centering
  \centerline{\includegraphics[width=.8\linewidth,keepaspectratio]{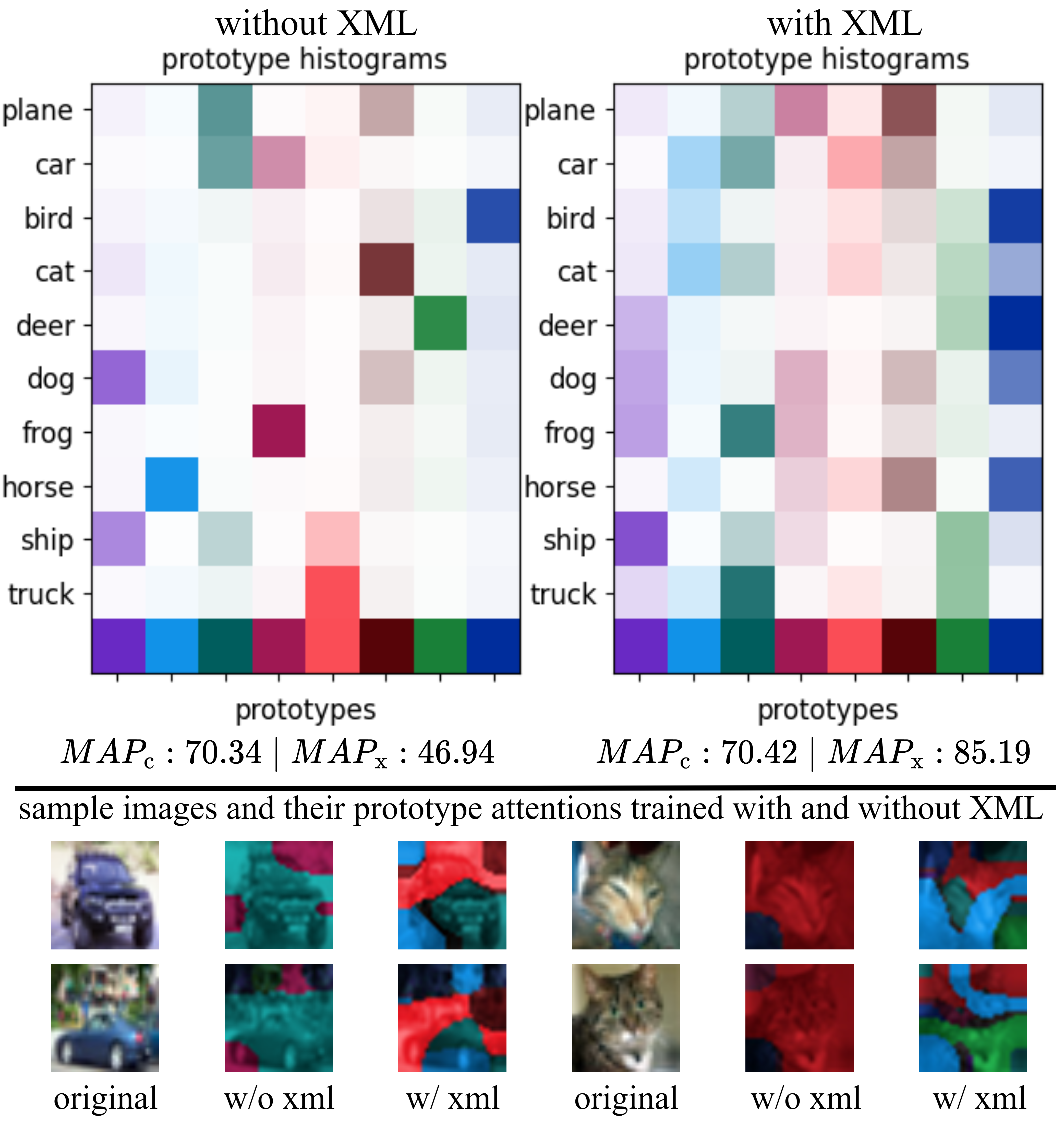}}
\caption{The distributions of the learned 8 prototypes across classes of Cifar10 dataset with and without $\mathcal{L}_{\text{XML}}$. Attention maps are coloured according to the dominant prototype at that location. $MAP_c$ denotes the performance of metric learning fitted to all classes and $MAP_x$ denotes the cross-class performance, \ie metric learning fitted to 5 classes and evaluated on the other 5 classes.}
\label{fig:xml}
\end{figure}
\section{Conclusion}
Building on the perspective explaining GAP as the convex combination of prototypes, we formulated learning of the prototypes and proposed cross-batch metric learning loss to regularize the learning for transferable prototypes. With extensive empirical studies, we validated the effectiveness of our method in various DML benchmarks.

\bibliographystyle{IEEEbib-abbrv}
\bibliography{main.bbl}

\appendix
\renewcommand{\theequation}{A.\arabic{equation}}
\section*{Appendix}
\label{sec:appendix}
\subsection*{Preliminaries}
\begin{apdefinition}[\textbf{Optimal Transport Distance}]\label{def:ot} The optimal transport (OT) distance between two probability mass distributions $(p,X)$ and $(q,Y)$ is:
\begin{equation} \label{eq:OT_distance}
\Vert (p,X) \shortminus (q,Y) \Vert_{OT} =   \!\!\!\min_{\substack{\pi\geqslant0 \\
\Sigma_i\pi_{ij}=q_j\\
\Sigma_{j}\pi_{ij}=p_i}}
 \!\!\textstyle\sum_{ij}c_{ij}\pi_{ij}
\end{equation}
where $c_{ij}=\Vert x_i \shortminus y_j \Vert_2$, and $(p,X)\in \Sigma_n \times \R^{d\times n}$ denotes a probability mass distribution with masses $p\in \Sigma_n$ in the probability simplex (\ie, $\Sigma_n \coloneqq \lbrace p \in \R^{n}_{\geqslant 0}\mid \textstyle\sum_i p_i = 1  \rbrace$), and $d$-dimensional support $X=[x_i]_{i\in[n]}\in\R^{d\sxtimes n}$.
\end{apdefinition}

\begin{apdefinition}[\textbf{Maximum Mean Discrepancy}] Max-imum mean discrepancy (MMD)  between two probability mass distributions $(p,X)$ and $(q,Y)$ is:
\begin{equation} \label{eq:MMD_distance}
\Vert (p,X) \shortminus (q,Y) \Vert_{MMD} =  \!\! \max_{f\in \mathcal{C}(X,Y)}
 \textstyle\sum_i p_i f(x_i) - \textstyle\sum_j q_j f(y_j)
\end{equation}
where $\mathcal{C}(X,Y)$ is the set of continuous and bounded functions defined on a set covering the column vectors of $X$ and $Y$.
\end{apdefinition}

\begin{apdefinition}[\textbf{Optimal Transport Distance Dual}] The Lagrangian dual of the optimal transport distance defined in \Cref{def:ot} reads:
\begin{equation} \label{eq:OT_distance_dual}
\Vert (p,X) \shortminus (q,Y) \Vert_{OT} =   \!\!\!\max_{f_i + g_j \leqslant c_{ij}}
 \!\!\textstyle\sum_i p_i f_i + \textstyle\sum_j q_j g_j
\end{equation}
with the dual variables $\lambda=\lbrace f,g\rbrace$.
\end{apdefinition}

Note that $x_i=y_j$ implies $f_i = -g_j$ and from the fact that $c_{ij}=c_{ji}$, we can express the problem in \eqref{eq:OT_distance_dual} as:
\begin{equation} \label{eq:OT_distance_dual_2}
\Vert (p,X) \shortminus (q,Y) \Vert_{OT} =   \max_{f\in \mathfrak{L}_1}
 \textstyle\sum_i p_i f(x_i) - \textstyle\sum_j q_j f(x_j)
\end{equation}
where $\mathfrak{L}_1=\lbrace  f \mid \sup\limits_{x,y}\tfrac{\vert f(x) - f(y)\vert}{\Vert x - y\Vert_2}\leqslant 1 \rbrace$ is the set of 1-Lipschitz functions.

\subsection*{Proofs}
\begin{apdefinition}[\textbf{Histogram Operator}] \label{def:hist_operatorA} For $n$-many $d$-dimensional features $X{=}[ x_i{\in}\R^d ]_{i=1}^n$ and $m$-many prototype features $\mathcal{V}=[ \nu_i{\in}\R^d ]_{i=1}^m$ of the same dimension, the histogram of $X$ on $\mathcal{V}$ is denoted as $z^\ast$ which is computed as the minimizer of the following problem:
\begin{equation} \label{eq:hist_operatorA}
(z^\ast,\pi^\ast) =   \!\argmax_{z\in\mathcal{S}^m, \pi\geqslant0}
 \!\textstyle\sum_{ij}\nu_i\T x_j\pi_{ij} \text{ s.to }
{\small\begin{array}[t]{l} \Sigma_i\pi_{ij}{=}\nicefrac{1}{n}\\[-1ex]\Sigma_{j}\pi_{ij}{=}z_i \end{array}}
\end{equation}
where $\mathcal{S}^m \coloneqq \lbrace p \in \R^{m}_{\geqslant 0}\mid \Sigma_i p_i = 1  \rbrace$.
\end{apdefinition}

\begin{apclaim}
The solution of the problem in \eqref{eq:hist_operatorA} reads:
\begin{equation}
    \pi^\ast_{ij} = \nicefrac{1}{n}\mathds{1}({i = \mathrm{argmax}_k\lbrace \nu_k\T x_j \rbrace})
\end{equation}
where $\mathds{1}(c)$ is 1 whenever $c$ is true and 0 otherwise.
\end{apclaim}

\begin{proof}
We prove our claim by contradiction. Denoting $c_{ij}= - \nu_i\T x_j$, for any $j$, we express a solution as $\pi^\ast_{ij}=\epsilon_i$ with $\epsilon_i \geqslant 0$ and $\sum_i \epsilon_i = \nicefrac{1}{n}$. Let $i^\ast = \argmin_k\lbrace c_{kj}\rbrace$. We can write $\pi^\ast_{i^\ast j}=\nicefrac{1}{n} - \sum_{i\mid i\neq i^\ast} \epsilon_i$. Our claim states that $\epsilon_i =0$ for $i\neq i^\ast$. We assume an optimal solution, $\pi^\prime$, with $\epsilon_i > 0$ for some $i\neq i^\ast$. Since $\pi^\prime$ is optimal, we must have $\sum_{ij}\pi^\prime_{ij}c_{ij}\leqslant \sum_{ij}\pi_{ij}c_{ij} $ for any $\pi$. For the $j^{th}$ column we have,
\begin{equation*}
\begin{split}
        \textstyle\sum_i\pi^\prime_{ij}c_{ij} &= (\tfrac{1}{n}-\textstyle\sum\limits_{i^\prime\mid i^\prime\neq i^\ast}\epsilon_{i^\prime}) c_{i^\ast j} + \textstyle\sum\limits_{i^\prime\mid i^\prime\neq i^\ast}\epsilon_{i^\prime} c_{i^\prime j} \\
        & = \tfrac{1}{n} c_{i^\ast j} + \textstyle\sum\limits_{i^\prime\mid i^\prime\neq i^\ast}\epsilon_{i^\prime} (c_{i^\prime j}-c_{i^\ast j}) \overset{(a)}{>} \textstyle\sum_i\pi^\ast_{ij}c_{ij}
\end{split}
\end{equation*}
where in $(a)$ we use the fact that $(c_{i^\prime j}-c_{i^\ast j}) > 0$ and $\epsilon_{i^\prime}>0$ for some $i^\prime$ by the assumption. Hence, $\textstyle\sum_{ij}\pi^\prime_{ij}c_{ij} > \textstyle\sum_{ij}\pi^\ast_{ij}c_{ij}$ poses a contradiction. Therefore, $\epsilon_{i^\prime}=0$ must hold for all $i^\prime \neq i^\ast$.
\end{proof}

\begin{aplemma}\label{thm:gap_protoA}
Given $n$-many convolutional features $X=[ x_i{\in}\mathcal{X}]_{i=1}^n$, and $m$-many prototype features $\mathcal{V}=[ \nu_i]_{i=1}^m$ with $\lbrace \nu_i\rbrace_{i=1}^m$ being $\delta$-cover of $\mathcal{X}$. If $z^\ast$ is the histogram  of $X$ on $V$, defined in \eqref{eq:hist_operatorA}, then we have:
$$ \Vert \textstyle\sum_{i=1}^m z^\ast_i\nu_i \shortminus \textstyle\sum_{j=1}^n \tfrac{1}{n} x_j \Vert_2 \leqslant \delta $$
\end{aplemma}

\begin{proof}
We can express
\[\Vert \textstyle\sum\limits_{i\in[m]} z^\ast_i\nu_i \shortminus \textstyle\sum\limits_{j\in[n]} \tfrac{1}{n} x_j \Vert_2^2 = \textstyle\sum\limits_{i\in[m]} p^\ast_i f(\nu_i) \shortminus \textstyle\sum\limits_{j\in[n]} q_j f(x_j) \]
where $f(x) = x\T(\textstyle\sum_i z^\ast_i\nu_i \shortminus \textstyle\sum_j \tfrac{1}{n} x_j)$, and $[n]=1,\ldots,n$. Note that 
$f$ is a continuous bounded operator for $\mathcal{X} =\lbrace x\mid \Vert x \Vert_2 \leqslant 1 \rbrace$ (We can always map the features inside unit sphere without loosing the relative distances). Moreover, the operator norm of $f$, \ie $\Vert f \Vert$, which is $\Vert \textstyle\sum_i z^\ast_i\nu_i \shortminus \textstyle\sum_j \tfrac{1}{n} x_j \Vert_2$ is less than or equal to 1. Thus, $f$ lie in the unit sphere of the continuous bounded functions set. Using the definition of MMD distance, we can bound the error as:
\[\textstyle\sum\limits_{i\in[m]} z^\ast_i f(\nu_i) \shortminus \textstyle\sum\limits_{j\in[n]} q_j f(x_j) \leqslant \Vert (z^\ast, V) \shortminus (q, X) \Vert_{MMD} \]
where $q_i=\nicefrac{1}{n}$ for all $i$. For the continuous and bounded functions of the operator norm less than 1, MMD is lower bound for OT \cite{sriperumbudur2010hilbert}. Namely,
\begin{equation*}
    \begin{split}
    \textstyle\sum\limits_{i\in[m]} z^\ast_i f(\nu_i) \shortminus \textstyle\sum\limits_{j\in[n]} q_j f(x_j) &\leqslant \Vert (z^\ast, V) \shortminus (q, X) \Vert_{MMD} \\
    &\leqslant \Vert (z^\ast, V) \shortminus (q, X) \Vert_{OT}    
    \end{split}
\end{equation*}
Since columns of $V$ is $\delta$-cover of the set $\mathcal{X}$, the optimal transport distance between the two distributions are bounded by $\delta$, \ie $\Vert (z^\ast, V) \shortminus (q, X) \Vert_{OT} \leqslant \delta$. Thus, we finally have:
\[\Vert \textstyle\sum\limits_{i\in[m]} z^\ast_i\nu_i \shortminus \textstyle\sum\limits_{j\in[n]} \tfrac{1}{n} x_j \Vert_2 \leqslant \delta.\]
\end{proof}

\end{document}